\documentclass[11pt]{article}

\usepackage[T1]{fontenc}
\usepackage[margin=1in]{geometry}
\usepackage{amsfonts,amsmath,amssymb,amsthm,blkarray,bm}
\usepackage[numbers,sort&compress]{natbib}
\usepackage[colorlinks=true, allcolors=blue]{hyperref}
\usepackage{booktabs}
\usepackage{graphicx}
\usepackage{algorithm}
\usepackage{algorithmic}
\usepackage{enumitem}

\newtheorem{assumption}{Assumption}

\newtheorem{lemma}{Lemma}

\newtheorem{theorem}{Theorem}

\def\R{\mathbb{R}}

\title{No Unique Minimizer, No Problem: \\
On the Consistency of Robust Neural Classifiers}

\renewcommand{\thefootnote}{\fnsymbol{footnote}}
\author{
Subhabrata Majumdar$^1$\footnotemark[1] \quad
Anand Deo$^1$\footnotemark[2] \quad
Partha Pratim Saha$^2$\footnotemark[2] \quad
Abhik Ghosh$^3$\\[3pt]
\shortstack{
\small{$^1$Indian Institute of Management Bangalore}\\
\small{$^2$Independent Researcher}\\
\small{$^3$Indian Statistical Institute, Kolkata}
}
}
\date{}

\begin{document}

\maketitle

\footnotetext[1]{Corresponding author, email: \texttt{smajumdar@iimb.ac.in}}
\footnotetext[2]{Equal contribution}
\renewcommand{\thefootnote}{\arabic{footnote}}

\begin{abstract}
Neural network classifiers trained by cross-entropy minimization are highly sensitive to label noise and adversarial contamination. While robust alternatives offer bounded influence and resistance to corruption, their statistical foundations in the deep learning setting are insufficient due to a fundamental difficulty: neural parameterizations are non-identifiable, so the population loss minimizer is an equivalence class of parameters, not a unique point. We develop a consistency theory for robust neural classifiers based on the S-divergence family that requires no identifiability assumption. Casting training as stochastic optimization over a non-identifiable parameter space, we prove that empirical S-divergence minimizers converge to the population-optimal equivalence class under mild regularity conditions, and verify these conditions for three architecture choices. We further establish that limit points of the robust training algorithm are stationary points of the empirical objective. Experiments on vision and language benchmark datasets confirm that S-divergence training maintains clean-data accuracy while exhibiting performance competitive with existing robust methods.
\end{abstract}

\section{Introduction}
\label{sec:introduction}

Modern neural networks achieve remarkable predictive performance across domains ranging from computer vision to natural language processing. However, their statistical foundations remain underexplored and challenging due to a fundamental mismatch between the classical theory of parametric estimation and the geometry of deep models. In traditional parametric statistics, consistency is usually established by proving that an estimator $\hat\theta_n \in \Theta$, obtained from a sample of size $n$, converges to a unique population parameter $\theta_0$. Such arguments rely on the identifiability assumption, in that different population parameter values must correspond to different probability distributions for the parameter estimate. Deep neural
networks violate this principle in an extreme manner. Permutation symmetries among hidden units, redundant neurons, scaling invariances, and overparameterization imply that many distinct parameter vectors generate exactly the same predictive function \citep{hechtnielsen1990algebraic,sussmann1992uniqueness}. This raises a fundamental statistical question:

\begin{quote}
\emph{Does the absence of a unique neural network parameter minimizer prevent statistical consistency?}
\end{quote}

In this work, we show that the answer is negative. The appropriate object of statistical inference is not the parameter vector itself, but the induced conditional probability model. We develop a theory of statistical consistency for a class of robust neural network classifiers $\{ g_\theta (\cdot): \theta \in \Theta \}$ that remains valid under parameter non-identifiability. Instead of requiring convergence of $\hat\theta_n$ to a single point, we prove convergence of empirical solutions to a population-level equivalence class
\[
\Theta_0=\{\theta: g_\theta(x)=p_0(x),\ x\in\mathcal{X}\},
\]
where all parameters in $\Theta_0$, plugged into the classifier $g_\theta$, represent the same Bayes-optimal conditional probability function $p_0$.

This is possible because neural network estimation is naturally a stochastic optimization problem over a possibly non-identifiable parameter space. Under mild regularity conditions ensuring uniform convergence of the empirical objective, every sequence of empirical maximizers approaches the population minimizer set, even when this set is not a singleton \citep{shapiro2021lectures}. Consequently, the learned classifier converges at the function level although the parameter sequence may oscillate indefinitely inside an equivalence class.

This perspective is particularly important for robust learning. Real-world datasets frequently contain label corruption, adversarial perturbations, and distributional contamination, under which likelihood-based training can become unstable. Robust divergence-based objectives, including density power divergence (DPD) and the broader $S$-divergence family, provide bounded influence and resistance against contaminated observations. However, existing theoretical analyses of these objectives in the neural network setting have proceeded under classical identifiability assumptions that deep models do not satisfy. A recent proposal \citep{jana2026rsdnet} proved Fisher consistency, Bayes-optimality, and robustness properties of multi-layer perceptrons (MLP) trained on the $S$-divergence loss, and demonstrated their empirical stability under label noise and adversarial perturbations. Their analysis, however, assumes a well-defined population minimizer, leaving open the question of whether these guarantees remain valid when the neural parameterization is non-identifiable.

The present work addresses this question directly. We propose a general framework for robust neural classifiers where parameter uniqueness is not assumed and statistical consistency is defined through convergence to the population minimizer set. This framework applies to modern architectures including MLPs, convolutional neural networks (CNNs), and Transformers.

\subsection{Contributions}
The proposed theory provides a bridge between classical M-estimation, robust statistics, and modern deep learning, establishing a principled foundation for robust neural classifiers beyond identifiable parametric models. Specifically

\begin{enumerate}[leftmargin=*,nolistsep]
    \item We prove that under mild regularity conditions, empirical $S$-divergence minimizers converge to the population-level equivalence class $\Theta_0$ without any parameter identifiability assumption (Theorem~\ref{thm:consistency}), and verify the regularity conditions for MLPs, CNNs, and Transformers (Lemma~\ref{lem:check_assumptions}).

    \item We propose a training algorithm to compute the empirical minimizer (Algorithm~\ref{alg:train}), and prove that limit points of this algorithm are stationary points of the empirical objective (Theorem~\ref{thm:stationarity}).

    \item We evaluate the $S$-divergence training procedure against eight competing losses on three computer vision benchmarks under benign and adversarial noise, confirming competitive clean-data accuracy and stability under contamination (Section~\ref{sec:experiments}).
\end{enumerate}

\subsection{Related Work}
\label{sec:literature}

Classical statistical learning theory has studied estimators under identifiable finite-dimensional models. Consistency of M-estimators was established through uniform laws of large numbers and argmax/argmin continuity arguments \citep{van2000asymptotic,newey1994large}. These results rely on convergence of empirical optimization problems to their population counterparts.

Deep neural networks fundamentally challenge this framework because parameterization is highly redundant in them~\citep{goodfellow2016deep}. Consequently, the parameter vector is generally not identifiable even when the predictive distribution is uniquely determined.
Recent work has investigated optimization landscapes of deep networks, showing that global minima often form connected manifolds of high dimensions rather than isolated points \citep{choromanska2015loss,nguyen2018loss}. However, these analyses primarily focus on optimization geometry and do not establish statistical consistency under non-identifiability. Our work differs by adopting the perspective of stochastic programming and set-valued estimation, where the target of inference is the population minimizer set rather than an individual parameter vector.

Modern deep learning systems are vulnerable to corrupted labels, distribution shifts, and adversarial examples. Extensive work has studied adversarial robustness and noise-resistant training strategies \citep{goodfellow2015explaining, madry2018towards}. Robust statistics provides a principled approach for accurate learning under such situations. The density power divergence loss \citep[DPD]{basu1998robust} enables estimators with bounded influence while retaining high efficiency under clean data. Its extensions, including the $S$-divergence family, generalize this robustness-efficiency tradeoff further \citep{ghosh2017generalized}.
Several robust losses have subsequently been adapted for machine learning, such as generalized cross entropy \citep{zhang2018}, symmetric cross entropy \citep{wang2019symmetric}, trimmed loss functions \citep{rusiecki2019trimmed}, DPD under the name of $\beta$-divergence \citep{ghosh2026provably}, $\alpha$-divergence \citep{wang2021alphanet}, $S$-divergence often via its transformed version known as $(\alpha, \beta)$-divergence \citep{cruces2026alpha}, and other noise-tolerant objectives.
The recent work of rSDNet \citep{jana2026rsdnet} provides a robust neural training procedure based on $S$-divergence minimization, demonstrating empirical stability under both label noise and adversarial perturbations. The present paper grounds this direction by providing the missing statistical theory for robust training techniques, with the $S$-divergence loss as an exemplar, which remains valid across modern deep learning architectures.

When population minimizers are not unique, classical convergence to a single estimator is replaced by convergence of solution sets. This framework appears naturally in stochastic programming and variational analysis \citep{shapiro2021lectures}. Under uniform convergence of empirical objectives, empirical minimizer sets converge to population minimizer sets through generalized argmin theorems.
Our analysis builds on these ideas and adapts them for robust neural networks. The resulting theory shows that deep learning does not require parameter identification for statistical consistency; instead, identification of the induced probability model is sufficient.

\section{Statistical Consistency of Neural Classifiers}
\label{sec:methodology}

Let $\mathcal D_n = (X_i,Y_i)_{i=1}^n$ be a dataset where $X_i\in \mathcal{X} \subset \R^k$ are i.i.d.\ random variables following the input distribution $P$, and $Y_i \in \{0,1\}$ be available labels. From this data, we wish to learn an underlying conditional model for probability $p(Y = 1 \mid X = x) := p(x)$. Suppose there exists a parametric form to this model:
\begin{equation}\label{eqn:nn_multi}
    p_\theta(x) = \frac{\exp(g_{\theta}(x))}{ 1+ \exp(g_{\theta}(x))}.
\end{equation}
Consider the problem of maximizing the (log)-likelihood of this model based on observed data:
\begin{align}\label{eqn:stoch_opt}
    \hat \theta_n =
    \arg\max_{\theta\in\Theta}
    \frac{1}{n}\sum_{i=1}^n
    \left[
        Y_i g_\theta(X_i)
        -
        \log\left\{1+\exp\left(g_\theta(X_i)\right)\right\}
    \right]
    := \arg\max_{\theta} \Psi_n(\theta).
\end{align}
Let the true conditional probability that needs to be learned equal $p_0(x)$, and define the true parameter set
\[
\Theta_0 = \left\{\theta\in \Theta : g_\theta(x) = p_0(x) , \forall x\in \mathcal{X}\right\}
\]
This is the maximiser of the population objective  $ \Psi(\theta) := \mathbb E[Y g_\theta(X) - \log(1+\exp(g_\theta(X)))] $ \citep[Sec.~5]{van2000asymptotic}.
Let $v^\star = \arg\min_{\theta\in \Theta_0} \Psi(\theta)$ denote the likelihood of models in $\Theta_0$.

\subsection{Convergence of the Solution}
\label{subsec:conv}
The objective is to show convergence without requiring identifiability (that is, for non-singleton $\Theta_0$). Note that the problem is a stochastic optimisation problem, where one needs to evaluate an optimal decision under uncertainty, captured by the distribution of the samples. Following \citet{shapiro2021lectures}, we make the following assumptions to establish regularity of solutions:

\begin{assumption}\label{assume:U-LLN1}
The samples $(X_i,Y_i) $ are i.i.d.
\end{assumption}

\begin{assumption}\label{assume:U-LLN2}
The function
    \[
    u(\cdot, x,y) = y g_\theta(x) - \log(1+\exp(g_\theta(x)))
    \]
is continuous at $\theta$ for almost every $(x,y)$.
\end{assumption}
An equivalent condition for Assumption~\ref{assume:U-LLN2} to hold is that $g_\theta(x)$ be continuous in $\theta$ for almost every $x$.
\begin{assumption}\label{assume:U-LLN3}
We have $|u(\theta,x,y)|\leq G(x,y)$ for all $\theta\in \Theta$ such that $\mathbb E_P[G(X,Y)]< \infty$.
\end{assumption}

Let $\Theta_n$ denote the maximizers of $\Psi_n$ over $\Theta$. We assume that this is non-empty for all large enough $n$.

\begin{theorem}\label{thm:consistency}
    Suppose that Assumptions~\ref{assume:U-LLN1}--\ref{assume:U-LLN3} hold, and the set $\Theta$ is compact. Then, with probability 1,
    \begin{equation}\label{eqn:convergence}
        \limsup_{n\to\infty } \Theta_n \subseteq \Theta_0 \text{ and } \Psi(\theta_n) \to v^\star,
    \end{equation}
for any selection $\theta_n\in \Theta_n$. In particular, if $\Theta_0 =\{\theta_0\}$ is a singleton, then  $\hat \theta_n\to \theta_0$ with probability 1.
\end{theorem}
\begin{proof}
    Assumptions~\ref{assume:U-LLN1}--\ref{assume:U-LLN3} imply the uniform convergence of the stochastic objective $\Psi_n$ defined in \eqref{eqn:stoch_opt} to the limit $\Psi$ over the set $\Theta$  \citep[Theorem~7.53]{shapiro2021lectures}. Thus, from \citet[Theorem~5.3]{shapiro2021lectures}, $\Psi_n(\theta_n) \to v^\star$ almost surely. Now,
    \[
    |\Psi(\theta_n) -v^\star| \leq \sup_{\theta\in \Theta}|\Psi_n(\theta) -\Psi(\theta)| + |\Psi_n(\theta_n) - v^\star|.
    \]
    Since $\Psi_n\to \Psi$ uniformly with probability 1 over $\Theta$, we further obtain
    $\Psi(\theta_n) \to v^\star$.  Further, given any sequence of solutions $\{\hat \theta_n\}_{n\geq 1}$ to the sample problem, $d(\hat \theta_n,\Theta_0) \to 0$ with probability 1, where $d(\theta,S) = \inf\{d(\theta,\theta^\prime) : \theta^\prime\in S\}$.
    Consequently, whenever a subsequence $\theta_{n_k}\in \Theta_{n_k}$ converges, its limit lies in $\Theta_0$. Equivalently, $\limsup_{n\to\infty} \Theta_n \subseteq \Theta_0$. The last assertion follows upon noting that when $\Theta_0 = \{\theta_0\}$, $d(\hat\theta_n,\Theta_0) = d(\hat \theta_n,\theta_0)$.
\end{proof}

\subsection{Network Architecture Specifications}
\label{sec:architectures}

We now specify the parametric form of $g_\theta$ in~\eqref{eqn:nn_multi} for
three standard architectures, and ensure that Assumptions~\ref{assume:U-LLN1}--\ref{assume:U-LLN3} are satisfied for them.  Throughout, let $z_0 = x \in \R^k$ denote the input, let $\theta_i \in \R^{d_i \times d_{i-1}}$ be the weight matrix at layer $i$ (with $d_0 = k$) of a neural network, and let $\phi_i : \R^{d_i} \to \R^{d_i}$ be an element-wise activation function applied after the $i$-th linear map. Define the pre-activation $s_i = \theta_i\, z_{i-1}$ and post-activation $z_i = \phi_i(s_i)$ at each layer. Following~\citet{rosati2026limits}, we write $D_{z_i} = \mathrm{diag}\!\bigl(\phi_i'(s_{i,1}),\ldots,\phi_i'(s_{i,d_i})\bigr)$ for the diagonal Jacobian of $\phi_i$ evaluated at $s_i$.

\paragraph{Multi-layer perceptron (MLP).}

An $L$-layer MLP computes
\begin{align*}
  g_\theta(x)
  &\;=\;
  \theta_{L+1}\,\phi_L\!\bigl(\theta_L\,\phi_{L-1}(\cdots\phi_1(\theta_1\,x)\cdots)\bigr)\\
  &\;=\;
(\theta_{L+1}\circ\phi_L\circ\theta_L\circ\cdots\circ\phi_1\circ\theta_1\,)(x),
\end{align*}
where $\theta_i\in\R^{d_i\times d_{i-1}}$ for $i=1,\ldots,L$ are hidden-layer
weight matrices and $\theta_{L+1}\in\R^{1\times d_L}$ is the classification head.  Bias
vectors $b_i\in\R^{d_i}$ may be absorbed into $\theta_i$ by appending a constant feature; we suppress them for notational clarity. The full parameter vector is $\theta=(\mathrm{vec}(\theta_1),\ldots,\mathrm{vec}(\theta_{L+1})) \in\Theta\subset\R^p$, with $p=\sum_{i=1}^{L+1}d_i\,d_{i-1}$.

For two adjacent hidden layers $i<j\le L$, the Hessian block
$\partial^2 g_\theta/\partial\theta_j\,\partial\theta_i$ decomposes into a product form $A\,B\,C$ \citep{rosati2026limits} where $B$ contains a
weight matrix $\theta_k$ ($i<k<j$) and $A,C$ are products of
activation Jacobians $D_{z_\ell}$ and Kronecker terms involving the input and
identity matrices.

\paragraph{Convolutional neural network (CNN).}
Let $\theta_i^{\mathrm{conv}} \in \R^{c_i\times c_{i-1}\times k_i\times k_i}$
denote a convolutional filter at layer $i$ with $c_i$ output channels,
$c_{i-1}$ input channels, and spatial kernel size $k_i$.  Define the
doubly-block-Toeplitz (or circulant, under periodic padding) matrix
$T(\theta_i^{\mathrm{conv}}) \in \R^{c_i h_i w_i \times c_{i-1} h_{i-1} w_{i-1}}$
that implements the linear convolution on a feature map of spatial size
$h_{i-1}\times w_{i-1}$.  The convolutional layer then has the same
linear-activation structure as the MLP:
\[
  s_i = T(\theta_i^{\mathrm{conv}})\,z_{i-1},
  \qquad
  z_i = \phi_i(s_i),
\]
so that $T(\theta_i^{\mathrm{conv}})$ plays the role of $\theta_i$ in the
generic composition $g_\theta(x) = \theta_{L+1}\circ\phi_L\circ
T(\theta_L^{\mathrm{conv}})\circ\cdots\circ\phi_1\circ
T(\theta_1^{\mathrm{conv}})\,x$.  Pooling operations (average or max) and
batch normalization, when present, are absorbed into $\phi_i$ or treated as
fixed linear maps between layers.  Because $T(\theta_i^{\mathrm{conv}})$ is a
structured linear operator, singular values of the Hessian blocks inherit the
same $A\,B\,C$ product structure, with $B = T(\theta_k^{\mathrm{conv}})^\top$
for an intermediate layer $k$.

\paragraph{Transformer (encoder block).}
Fix a sequence of $N$ token embeddings
$Z_0 = [z_0^{(1)},\ldots,z_0^{(N)}]^\top \in \R^{N\times d}$.  A single
Transformer encoder block~\citep{vaswani2017attention} with $H$ heads computes
\begin{align}
  \mathrm{Attn}(Z_0)
  &= \sum_{h=1}^{H}
     \mathrm{softmax}\!\Bigl(\frac{Z_0\,W_Q^{(h)}\bigl(W_K^{(h)}\bigr)^\top Z_0^\top}
     {\sqrt{d_h}}\Bigr)
     \times Z_0\,W_V^{(h)}\,W_O^{(h)},
     \label{eqn:mha}\\
  Z_1 &= \mathrm{LN}\!\bigl(Z_0 + \mathrm{Attn}(Z_0)\bigr),
     \label{eqn:res1}\\
  Z_2 &= \mathrm{LN}\!\bigl(Z_1 + \theta_2\,\phi(\theta_1\,Z_1^\top)\bigr),
     \label{eqn:ffn}
\end{align}
where $W_Q^{(h)}, W_K^{(h)}, W_V^{(h)} \in \R^{d\times d_h}$ and
$W_O^{(h)} \in \R^{d_h\times d}$ are the query, key, value, and output
projection matrices for head $h$, with $d_h = d/H$;
$\theta_1\in\R^{d_{\mathrm{ff}}\times d}$ and
$\theta_2\in\R^{d\times d_{\mathrm{ff}}}$ are the feed-forward network (FFN)
weight matrices; $\phi$ is the FFN activation (e.g.\ GELU); and $\mathrm{LN}$
denotes layer normalization.

When stacking $L$ such blocks, the full parameter vector $\theta$ collects
$$ \left\{ W_Q^{(h,\ell)}, W_K^{(h,\ell)}, W_V^{(h,\ell)}, W_O^{(h,\ell)},
\theta_1^{(\ell)}, \theta_2^{(\ell)} \right\}_{\ell=1}^{L}, $$
together with an initial embedding matrix $E\in\R^{|\mathcal V|\times d}$ and the final
classification head.  Within the FFN sub-block of each layer, the pair
$(\theta_1^{(\ell)}, \theta_2^{(\ell)})$ gives a two-layer MLP whose Hessian
blocks admit the same $A\,B\,C$ factorization: concretely,
$\partial^2 g_\theta/\partial\theta_2^{(\ell)}\,\partial\theta_1^{(\ell)}$
yields $B = \theta_1^{(\ell)\top}$ (or $\theta_2^{(\ell)\top}$) flanked by
activation-Jacobian and input-dependent factors $A$ and $C$.  For the
self-attention parameters, the Hessian blocks
$\partial^2 g_\theta/\partial\mathrm{vec}(W_V^{(h)})\,
\partial\mathrm{vec}(W_Q^{(h')})$ involve products of the attention weight
matrix $\mathrm{softmax}(\cdot)$, the input $Z_0$, and the projection
matrices, again yielding a matrix-product structure whose singular values are
controlled by the spectral norms of the constituent weight matrices.

\paragraph{Assumption check.}
In all three architectures, the classifier $p_\theta(x)$ of~\eqref{eqn:nn_multi} is obtained by applying the logistic link to the scalar output $g_\theta(x)$, and the parameter space $\Theta$ is the Cartesian product of the matrix spaces for all weight matrices.  The consistency result of Theorem~\ref{thm:consistency} applies to any
$g_\theta$ satisfying Assumptions~\ref{assume:U-LLN1}--\ref{assume:U-LLN3}, which each of these architectures does under standard regularity (bounded inputs, smooth or piecewise-smooth activations with integrable envelope). We capture this in Lemma~\ref{lem:check_assumptions} below.
\begin{lemma}\label{lem:check_assumptions}
Suppose $\Theta$ is compact and $\mathbb E_P[\|X\|]<\infty$. Let $\phi_i$ be continuous and there exist a $C_0$ such that $\max_{i}|\phi_i(t)|<C_0(1+|t|)$ for all $t$.
    Then, Assumptions~\ref{assume:U-LLN2}--\ref{assume:U-LLN3} are satisfied in each of the neural architectures above.
\end{lemma}

\noindent \textit{Proof sketch.}
For each architecture, we verify two properties: (a)~continuity of $g_\theta(x)$ in $\theta$ for every $x$ (see the comment following Assumption~\ref{assume:U-LLN2}), and (b)~existence of an integrable envelope $G(x,y)$ dominating $|u(\theta,x,y)|$ uniformly over $\Theta$ (Assumption~\ref{assume:U-LLN3}).

For MLPs, (a) follows from the fact that $g_\theta$ is a finite composition of continuous maps (linear maps and activations). For (b), the growth condition $|\phi_i(t)| \le C_0(1+|t|)$ and compactness of $\Theta$ yield, by induction over layers, a bound $\sup_{\theta \in \Theta}|g_\theta(x)| \le a + b\|x\|$ for finite constants $a,b$ depending only on the architecture and $\sup_{\theta \in \Theta}\|\theta_i\|$. The envelope $G(x,y) = 2(a+b\|x\|) + \log 2$ is integrable by the first moment condition on $X$.

For CNNs, (a) and (b) follow identically once we observe that the map from convolutional filter weights to the corresponding Toeplitz matrix is linear (hence continuous), so each convolutional layer has the same linear-activation structure as the MLP case.

For Transformers, (a) holds because every constituent operation---the bilinear query-key product, softmax, linear projections, FFN activations, and layer normalization---is continuous in $\theta$. For (b), the key observation is that layer normalization produces outputs with uniformly bounded norm ($\|\mathrm{LN}(v)\| \le \Gamma$ for a constant $\Gamma$ depending only on the embedding dimension), and since each encoder block terminates in layer normalization, the bound does not accumulate across layers. This yields a \emph{constant} envelope $G(x,y) = 2M_0\Gamma + \log 2$, so Assumption~\ref{assume:U-LLN3} holds without any moment condition on $X$. The full proof is given in Appendix~\ref{app:proof-lemma1}.
\hfill$\square$

\section{Computational Algorithm}
\label{sec:algorithm}

We now describe a training algorithm for robust neural classifiers based on S-divergence minimisation. For a $J$-class classification problem with training data $\mathcal{D}_n$, where $Y_i \in \{e_1, \ldots, e_J\}$ are one-hot encoded labels and the network outputs class probabilities $p(X_i; \theta) = (p_1(X_i; \theta), \ldots, p_J(X_i; \theta))^\top$ via a softmax layer, the empirical S-divergence loss \citep{ghosh2017generalized,jana2026rsdnet} is
\begin{align}\label{eqn:sdloss}
  L_{\beta,\lambda}^{\mathrm{SD}}(\theta)
  &=
  \frac{1}{n}\sum_{i=1}^n
  \left\{
    \frac{1}{B}\sum_{j=1}^{J} p_j(X_i;\theta)^{1+\beta}
    - \frac{1+\beta}{B}\prod_{j=1}^{J} p_j(X_i;\theta)^{y_{ij}} \right\},
\end{align}
where $B = \beta - \lambda(1-\beta)$, with validity constraints $A \equiv 1 + \lambda(1-\beta) > 0$ and $B > 0$. Setting $\beta = \lambda = 0$ recovers the categorical cross-entropy (CCE) loss, and $\lambda = 0$ reduces to the density power divergence (DPD) loss. The tuning parameters $(\beta,\lambda)$ control the robustness--efficiency trade-off: larger $\beta$ increases down-weighting of outlying observations at the cost of statistical
efficiency.

The loss in Eq.~\eqref{eqn:sdloss} depends only on the network weights $\theta$ through the softmax probabilities, yielding a single-loop stochastic gradient procedure. This simplification holds for all architectures considered in Section~\ref{sec:methodology}---MLPs, CNNs, and Transformers---since the loss interacts with the architecture only through the predicted probability vector $p(x;\theta)$.

Algorithm~\ref{alg:train} gives the training procedure. The algorithm is architecture-agnostic: the forward pass in Step~5 computes $p(X_i; \theta)$ using whichever architecture $g_\theta$ is specified, followed by a softmax output layer. Backpropagation in Step~7 computes gradients through the same computational graph via automatic differentiation. The only component specific to robust training is the loss computation in Step~6, which replaces the standard cross-entropy with the S-divergence objective~\eqref{eqn:sdloss}.

\begin{algorithm}[t]
\caption{Robust Training for Neural Network Classifiers}
\label{alg:train}
\begin{algorithmic}[1]
\REQUIRE Training data $\mathcal{D}_n = \{(X_i, Y_i)\}_{i=1}^n$;
  network $g_\theta$;
  S-divergence parameters $(\beta, \lambda)$;
  learning rate $\eta$; batch size $m$; number of epochs $T$
\ENSURE Trained parameters $\hat\theta$
\STATE Initialize $\theta^{(0)}$ randomly
\FOR{$t = 0, 1, \ldots, T-1$}
  \STATE Randomly partition $\mathcal{D}_n$ into mini-batches
    $\mathcal{B}_1, \ldots, \mathcal{B}_{\lceil n/m \rceil}$
    of size $m$
  \FOR{each mini-batch $\mathcal{B}$}
    \STATE Compute class probabilities
      $p(X_i; \theta^{(t)})$ via forward pass with softmax output,
      for all $(X_i, Y_i) \in \mathcal{B}$
    \STATE Compute mini-batch S-divergence loss $ \hat L_{\beta,\lambda}^\mathrm{SD} (\theta^{(t)}) $
    \STATE Compute gradient
      $\nabla_\theta \hat L_{\beta,\lambda}^{\mathrm{SD}}(\theta^{(t)})$
      via backpropagation
    \STATE Update: $\theta^{(t+1)} \leftarrow
      \mathrm{Adam}\!\left(\theta^{(t)},\,
      \nabla_\theta \hat L_{\beta,\lambda}^{\mathrm{SD}}(\theta^{(t)}),\,
      \eta\right)$
  \ENDFOR
\ENDFOR
\RETURN $\hat\theta = \theta^{(T)}$
\end{algorithmic}
\end{algorithm}

The following result establishes that the iterates of Algorithm~\ref{alg:train} converge to stationary points of the empirical objective. Let $\{\theta^{(t)}\}_{t \ge 1}$ denote the sequence of iterates, and let $\theta_0 \in \Theta_0$ be a member of the true parameter set.

\begin{theorem}
\label{thm:stationarity}
Fix the training sample $\mathcal{D}_n = (X_i, Y_i)_{i=1}^n$. Suppose the iterates
$\{\theta^{(t)}\}_{t\ge1}$ of Algorithm~\ref{alg:train} are uniformly bounded
by a quantity depending only on the sample size and network architecture:
\[
  \big\| \theta^{(t)} - \theta_0 \big\| \;\le\; R(n, \dim\Theta),
  \qquad t \ge 1 .
\]
Then any limit point $\theta^{\infty}$ of $\{\theta^{(t)}\}$ is a stationary
point of $L_{\beta,\lambda}^{\mathrm{SD}}(\cdot)$, i.e.\ a point at which
\[
  \nabla_\theta L_{\beta,\lambda}^{\mathrm{SD}}(\theta^{\infty}) = 0.
\]
\end{theorem}

\begin{proof}
The proof follows the structure of Theorem~5 of \citet{majumdar22jmmle}, which generalizes Theorem~1 of \citet{linetal16}, substituting $\big(\hat\theta^{(t)}, \theta_0\big)$ and
$(\theta^\infty)$ for the pairs $\big(\hat B^{(t)},\hat\Theta_y^{(t)}\big)$,
$(B_0,\Theta_{y0})$ and $(B^\infty,\Theta_y^\infty)$ therein, with the
corresponding modifications below.

Since Assumptions~\ref{assume:U-LLN2}--\ref{assume:U-LLN3} ensure
$u(\theta,x,y) = y\,g_\theta(x) - \log\{1+\exp(g_\theta(x))\}$ is
continuous in $\theta$ for a.e.\ $(x,y)$ and dominated by an integrable
envelope, the same regularity carries over to the pointwise S-divergence
contribution
\[
  v_\theta(x,y) \;=\;
  \frac{1}{B}\sum_{j=1}^{J} p_j(x;\theta)^{1+\beta}
  - \frac{1+\beta}{B}\prod_{j=1}^{J} p_j(x;\theta)^{y_j},
\]
which is a smooth (indeed $C^\infty$) function of $g_\theta(x)$ for fixed
$\beta>0$. Hence $L_{\beta,\lambda}^{\mathrm{SD}}(\theta) = n^{-1}\sum_i
v_\theta(X_i,Y_i)$ is continuously differentiable in $\theta$ whenever
$g_\theta$ is (e.g.\ for a network with smooth activations).

By construction, Algorithm~\ref{alg:train} is a descent algorithm for
$L_{\beta,\lambda}^{\mathrm{SD}}$: each iterate satisfies
$L_{\beta,\lambda}^{\mathrm{SD}}(\theta^{(t+1)}) \le
L_{\beta,\lambda}^{\mathrm{SD}}(\theta^{(t)})$,
with equality only at stationary points. The
boundedness assumption $\|\theta^{(t)} - \theta_0\| \le R(n,\dim\Theta)$
implies $\{\theta^{(t)}\}$ lies in a compact subset of $\Theta$, so it admits
at least one limit point $\theta^{\infty}$, and some subsequence
$\theta^{(t_k)} \to \theta^{\infty}$.

Because $L_{\beta,\lambda}^{\mathrm{SD}}$ is continuous and monotonically
non-increasing along the sequence, and bounded below by $0$ (the
S-divergence is non-negative by Theorem~3.1 of \citealp{ghosh2017generalized}),
$L_{\beta,\lambda}^{\mathrm{SD}}(\theta^{(t)})$ converges to some limit
$c^\infty \ge 0$, and by continuity
$L_{\beta,\lambda}^{\mathrm{SD}}(\theta^{(t_k)}) \to
L_{\beta,\lambda}^{\mathrm{SD}}(\theta^{\infty}) = c^\infty$. Standard
arguments for descent algorithms (Zangwill's global convergence theorem;
cf.\ the proof of Theorem~1 in \citealt{linetal16}) then show that if
$\theta^\infty$ were not a stationary point, continuity of
$\nabla_\theta L_{\beta,\lambda}^{\mathrm{SD}}$ together with the descent
property would produce a further strict decrease in a neighborhood of
$\theta^\infty$ along the subsequence, contradicting convergence of
$L_{\beta,\lambda}^{\mathrm{SD}}(\theta^{(t)})$ to $c^\infty$. Hence
$\nabla_\theta L_{\beta,\lambda}^{\mathrm{SD}}(\theta^{\infty}) = 0$.
\end{proof}

Theorem~\ref{thm:stationarity} places no
explicit rate requirement on $R(n,\dim\Theta)$: any finite bound on the
iterates suffices for existence of a stationary limit point. A tighter bound,
e.g.\ $R(n,\dim\Theta) = O_P(\sqrt{\dim\Theta \log n / n})$ obtained from the
finite-sample robustness bounds for minimum S-divergence estimators
\citep{ghosh2015continuous,ghosh2017generalized}, additionally ensures
that the resulting stationary point lies close to $\Theta_0$ with high
probability.

\section{Experiments}
\label{sec:experiments}

We evaluate the proposed S-divergence (SDIV) loss against eight other loss functions and three classification benchmarks, with the goal of characterising the accuracy of SDIV-trained models on clean data and their robustness under (i)~uniform label noise, (ii)~FGSM adversarial perturbations, and (iii)~variation of the SDIV tuning parameters $(\beta,\lambda)$.

\subsection{Setup}
\label{sec:setup}

To evaluate the performance of our method, we perform experiments on six benchmark datasets from two domains: four from computer vision: MNIST, CIFAR10, PathMNIST, and DermaMNIST~\citep{MedMNIST-PathMNIST+DermaMNIST}, and two from NLP: Emotion~\citep{saravia2018} and PubMedQA~\citep{PubMedQA}. Nine classification objectives are compared: SDIV, standard categorical cross-entropy (CCE) as a non-robust baseline, mean absolute error (MAE) on predicted probabilities, Generalised Cross-Entropy~\citep[GCE]{zhang2018}, GCE applied only to samples with $p_y < 0.5$, i.e.\ confidence masking (TruncGCE), Symmetric Cross-Entropy~\citep{wang2019symmetric},  Truncated DPD combined with symmetric CCE (TDPD-CCE), Truncated Symmetric Categorical Cross-Entropy: equivalent to SDIV at $\lambda{=}0$ (TSCCE), and Fractional Classification Loss \citep[FCL]{kurucu2025introducingfractionalclassificationloss}.

As the model architecture we use a vanilla Transformer encoder paired with the SDIV loss whose architecture mirrors the patch-based ViT design~\citep{dosovitskiy2021} at small scale. We decompose MNIST and CIFAR-10 images into $7{\times}7$ non-overlapping patches, and resize PathMNIST and DermaMNIST images to $64{\times}64$ then decompose into $8{\times}8$ patches. All vision models are trained with the Adam optimizer ($\eta_0{=}10^{-3}$, $\beta_1{=}0.9$, $\beta_2{=}0.999$) for 30 epochs, batch size 256.
For NLP tasks we fine-tune \texttt{bert-base-uncased}~\citep{devlin2019} with a linear classification head, using Adam ($\eta_0{=}2{\times}10^{-5}$) with batch size 32 for 3 epochs.

\subsection{Clean Data Performance}
\label{sec:clean}
We begin with a tuning-parameter sweep for SDIV to quantify performance variation across values of $(\beta,\lambda)$. The results in Table~\ref{tab:sdiv_surface_paper} indicate that lower values of $\beta$ ($\beta \leq 0.2$) and higher values of $\lambda$ ($\lambda \geq -0.4)$ lead to both algorithmic stability and better performance. Table~\ref{tab:clean_med} presents a comparison of test accuracy for the tuned SDIV model with other methods. SDIV narrowly achieves highest accuracy on both datasets. However, the narrow between-method differences indicate that a number of the methods achieve accuracy on both datasets that are statistically similar to CCE, the non-robust alternative. In DermaMNIST, the performance of three methods (GCE, TruncGCE, MAE) is the same as the prevalence of the majority-class in the dataset (66.88\%), indicating convergence issues leading to constant prediction across all samples.

\begin{table}[ht]
\centering
\setlength{\tabcolsep}{3pt}
\begin{tabular}{ccrr}
\toprule
$\beta$ & $\lambda$ & PathMNIST & DermaMNIST \\
\midrule
0.02 & $-$0.80 & 77.87               & 66.88$^{\dagger}$ \\
0.02 & $-$0.40 & 83.23               & 71.42             \\
0.02 & $+$0.00 & 83.68               & 72.87             \\
\midrule
0.05 & $-$0.80 & 82.60   & 66.88$^{\dagger}$ \\
0.05 & $-$0.40 & \textbf{84.11}      & 71.57             \\
0.05 & $+$0.00 & 82.10               & 72.32             \\
\midrule
0.10 & $-$0.80 & 81.52               & 66.88$^{\dagger}$ \\
0.10 & $-$0.40 & 83.69               & \textbf{73.32}    \\
0.10 & $+$0.00 & 82.06               & 72.07             \\
\midrule
0.20 & $-$0.80 & 82.45               & 67.13             \\
0.20 & $-$0.40 & 82.41               & 72.87             \\
0.20 & $+$0.00 & 83.33               & 71.77             \\
0.20 & $+$0.20 & 82.59               & 71.57             \\
\midrule
0.50 & $-$0.80 & 81.66               & 66.88$^{\dagger}$ \\
0.50 & $-$0.40 & 82.62               & 70.52             \\
0.50 & $+$0.00 & 82.63               & 71.12             \\
0.50 & $+$0.20 & 82.60               & 71.47             \\
\bottomrule
\end{tabular}%
\caption{%
SDIV $(\beta,\lambda)$ tuning surface: test accuracy (\%) on PathMNIST and DermaMNIST, with \textbf{Bold} indicating best per dataset. $^{\ddagger}$Values equal the DermaMNIST majority-class prevalence, indicating a degenerate constant-class predictor.}
\label{tab:sdiv_surface_paper}
\end{table}
\begin{table}[ht]
\centering
\setlength{\tabcolsep}{3pt}
\begin{tabular}{lcccc}
\toprule
Loss & MNIST & CIFAR-10 & DermaMNIST & PathMNIST \\
\midrule
CCE                  & \underline{98.19}          & \underline{60.56}       & \underline{73.22}     & 83.02                      \\
MAE                  & --                         & --                      & 66.88$^{\ddagger}$    & 78.50                      \\
GCE      & 98.36                      & 56.52                   & 66.93$^{\ddagger}$    & 82.24                      \\
TruncGCE             & --                         & --                      & 66.93$^{\ddagger}$    & 78.20                      \\
SCE                  & 98.18                      & 59.05                   & 70.22                 & 83.06                      \\
TDPD-CCE             & \textbf{98.50}             & \textbf{61.16}          & 72.32                 & 82.10                      \\
TSCCE                & 86.23                      & 52.94                   & 70.82                 & 82.26                      \\
FCL                  & 9.80$^{\ddagger}$          & 59.06                   & 72.57        & \underline{83.61}             \\
SDIV (ours) & 97.94                      & 55.06                   & \textbf{73.32}          & \textbf{84.11}   \\
\bottomrule
\end{tabular}%
\caption{Test accuracy (\%) under clean labels. \textbf{Bold}, \underline{underlined}, and $^{\ddagger}$ values indicate best, second best, and constant-class predictions, respectively.}
\label{tab:clean_med}
\end{table}

\subsection{Robustness to Noisy Labels}
\label{sec:noise}

We inject noise in the image labels by randomly selecting $\eta\%$ of training images and flipping their labels, with $\eta \in \{0,10,20,30,40\}$. Figure~\ref{fig:noise} shows the results.
On MNIST, SDIV and GCE are the most noise-tolerant losses. The narrow differences of SDIV accuracy values between $\eta=0$ and $\eta=40$ are consistent with the theoretical prediction that the SDIV loss down-weights corrupted samples.
On DermaMNIST, we see the recurring pattern of decreasing accuracy with increasing noise labels. SDIV has convergence problems at intermediate noise levels, which it recovers from and reaches highest accuracy at $\eta = 40$.
On PathMNIST, most losses maintain stable accuracy across noise rates, except MAE. The near-flat trajectories indicate a ceiling effect on this strongly structured dataset. MAE is a notable outlier: at $\eta{=}0.1$ it collapses to 48.76\% and only partially recovers at higher noise levels.

\begin{figure}[ht]
    \centering
    \includegraphics[width=\linewidth]{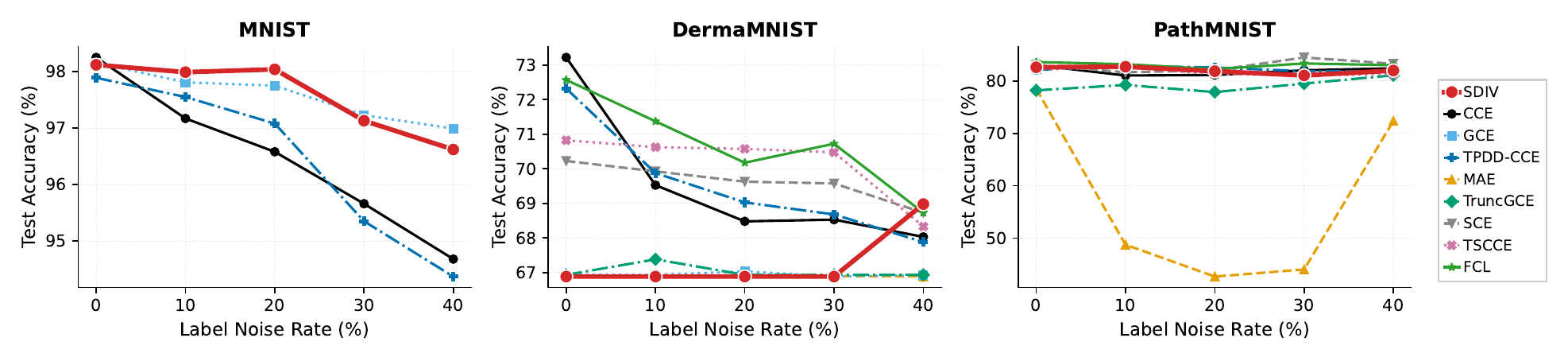}
    \caption{Test accuracy (\%) under uniform label noise. Missing methods in each subplot gave constant prediction.}
    \label{fig:noise}
\end{figure}
\begin{figure}[ht]
    \centering
    \includegraphics[width=\linewidth]{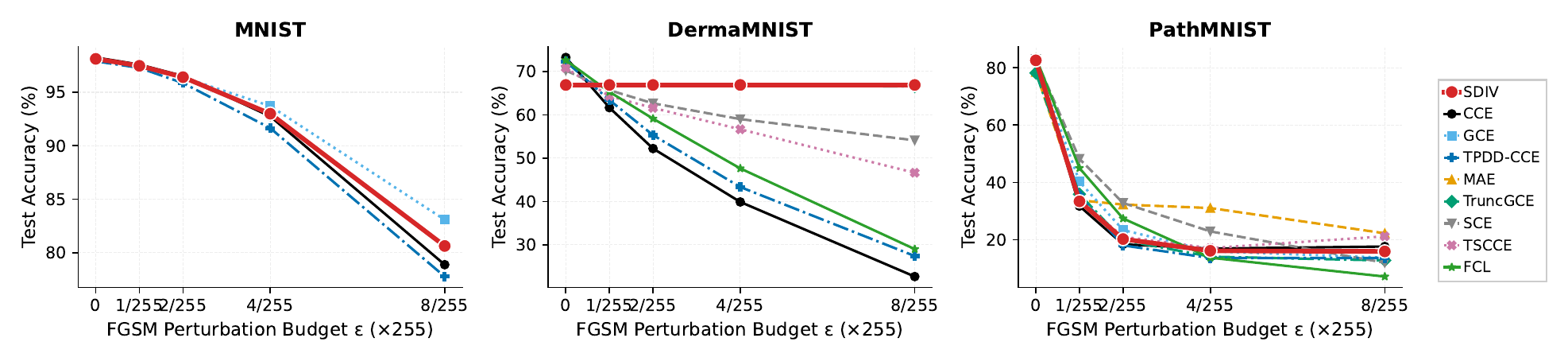}
    \caption{Test accuracy (\%) under adversarial perturbations. Missing methods in each subplot gave constant prediction.}
    \label{fig:fgsm}
\end{figure}

\subsection{Adversarial Robustness}
\label{sec:advrob}

We construct adversarial examples using the Fast Gradient Sign Method~\citep[FGSM]{goodfellow2015explaining}, with perturbation budgets $\varepsilon \in \{0, 1, 2, 4, 8\}/255$. Gradients are computed using the CCE loss uniformly (to prevent information leakage from the training loss into the attack). Figure~\ref{fig:fgsm} summarises the results.
Under FGSM at $\varepsilon{=}8/255$, GCE is the most robust on MNIST, followed by SDIV (80.63\%). SDIV has highest accuracy in DermaMNIST. In general, all methods retain high accuracy relative to their clean baselines. PathMNIST is far more vulnerable to the FGSM attack. At $\varepsilon{=}8/255$, all losses show catastrophic degradation. No loss provides genuine resistance to FGSM on PathMNIST, indicating that robust training losses alone without explicit adversarial training~\citep{madry2018towards}, do not confer adversarial robustness.

\subsection{Robust Fine-tuning on Text Data}
\label{sec:nlp}

Emotion~\citep{saravia2018} is a 6-class emotion recognition task, and PubMedQA~\citep{PubMedQA} is a 3-class biomedical question-answering task. Table~\ref{tab:nlp_results} reports the validation accuracy for our methods. SDIV attains the highest accuracy on both datasets, narrowly ahead of the strongest baselines. Also notable is the stability of SDIV relative to the other robust losses: several methods (SCE and TDPD-CCE on both datasets, GCE and FCL on PubMedQA) fail to converge to a non-trivial predictor, whereas SDIV trains stably on both tasks. This pattern is consistent with the vision experiments, where SDIV in its recommended parameter range avoids the gradient-starvation collapse that affects more aggressive robust losses. This supports the paper's central claim: because the consistency guarantee is architecture-agnostic, the same S-divergence objective transfers from convolutional and vision-Transformer classifiers to a pretrained language model without modification.

\begin{table}[t]
\centering
\setlength{\tabcolsep}{3pt}
\begin{tabular}{lcc}
\toprule
Loss & Emotion & PubMedQA \\
\midrule
CCE                  & 57.25                      & 56.00          \\
MAE                  & 57.60                      & 55.33          \\
GCE          & 58.20                      & --             \\
TruncGCE             & 58.10                      & 58.00 \\
SCE                  & --                          & --             \\
TDPD-CCE             & --                          & --             \\
TSCCE                & 57.80                      & 56.00          \\
FCL                  & 57.70                      & --             \\
SDIV (ours)          & \textbf{58.50}              & \textbf{58.67} \\
\bottomrule
\end{tabular}%
\caption{Validation accuracy (\%) on NLP datasets. \textbf{Bold} indicates best, and ``--'' in a cell indicates convergence issues for that method $\times$ dataset.}
\label{tab:nlp_results}\end{table}

\section{Discussion}
\label{sec:discussion}

The consistency framework of Section~\ref{sec:methodology} guarantees that any sequence of empirical S-divergence minimizers converges to the population-optimal equivalence class, and Lemma~\ref{lem:check_assumptions} confirms that this guarantee holds for MLPs, CNNs, and Transformers under standard regularity conditions. The experiments validate that this theoretical machinery connects to observable behavior in practice.

\paragraph{Scope of theoretical results.}
The consistency theorem is stated for binary classification, whereas the experiments use multiclass softmax classifiers. The extension is mostly notational: the uniform-convergence and envelope conditions of Lemma~\ref{lem:check_assumptions} continue to hold under the multinomial likelihood, since our assumptions (continuity and boundedness of $g_\theta$) are unaffected by the number of output classes. Second, the stationarity result (Theorem~\ref{thm:stationarity}) assumes the iterates remain in a bounded region. This holds in practice under weight decay or projected updates, but establishing it from first principles for Adam on these architectures is a substantive extension we leave to future work.

\paragraph{Relation to loss-landscape geometry.}
The equivalence class $\Theta_0$ is closely related to the connected minima manifolds studied in the loss-landscape literature \citep{choromanska2015loss,nguyen2018loss}. Those works characterize the geometry of the set of global minimizers of the training objective. Our contribution is complementary, as we show that empirical minimizers of the population loss concentrate on $\Theta_0$ regardless of its shape.

\paragraph{A note on layer normalization.}
The Transformer verification in Lemma~\ref{lem:check_assumptions} yields a stronger conclusion than the MLP and CNN cases, since each encoder block terminates in layer normalization, the envelope $G(x,y)$ is constant, and Assumption~\ref{assume:U-LLN3} holds with no moment condition on the input distribution. The recursive bounds needed for MLPs and CNNs have no analogue here, since the per-block output bound does not depend on the block's input. This suggests that beyond optimization benefit, layer normalization confers a statistical regularity that makes Transformers amenable to this style of analysis.

\paragraph{Limitations and Future Work.}
A few aspects of the experimental evaluation warrant caution. The adversarial evaluation is limited to FGSM, a single-step attack; stronger attacks such as PGD would provide a more rigorous assessment. On the theoretical side, the compactness assumption on $\Theta$ is standard in M-estimation theory but may not hold exactly for unconstrained gradient descent. Weight decay or projected gradient methods restore it in practice, but a formal treatment of the unbounded case via local compactification arguments would be valuable.

Our experiments confirm that S-divergence training is competitive with cross-entropy on clean data and degrades more gracefully under label noise, while also revealing the practical importance of tuning parameter selection and the limits of statistical robustness against adversarial attack. Among directions of future work, finite-sample convergence rates connecting the consistency result to explicit bounds on $d(\hat\theta_n, \Theta_0)$ as a function of $n$, $\dim\Theta$, and $(\beta,\lambda)$ would provide practical guidance on sample complexity. Combining S-divergence training with adversarial training to achieve simultaneous robustness to distributional contamination and worst-case input perturbation is a natural next step. Finally, relaxing the compactness assumption on $\Theta$ to accommodate unconstrained optimization, perhaps through implicit regularization arguments, would bring the theory closer to standard deep learning practice.

\section{Acknowledgements}
This research is supported by Indian Institute of Management Bangalore Research Seed Grant R\&P242-71.

\bibliographystyle{abbrvnat}
\bibliography{references}

\appendix

\section{Proof of Lemma~\ref{lem:check_assumptions}}
\label{app:proof-lemma1}

\begin{proof}
\noindent \textbf{i) Verification for Multilayer Perceptrons.}
It is sufficient to show that $g_\theta(x)$ is continuous in $\theta$ for
every $x$, and that $\sup_{\theta\in\Theta}|g_\theta(x)|$ admits an
integrable envelope. Since each linear map $\theta_i \mapsto \theta_i z$
and each activation $\phi_i$ is continuous, so is their composition; hence
$g_\theta(x)$ is continuous in $\theta$ for every $x$. This verifies
Assumption~\ref{assume:U-LLN2}.

To verify the uniform bound, note that the componentwise growth condition
yields $\|\phi_i(z)\| \le C_0\bigl(\sqrt{d_i} + \|z\|\bigr)
\le C_0(K_0 + \|z\|)$, where $K_0 := \max_i \sqrt{d_i}$. Writing the
output of the $i$th layer as $h_i^\theta(x) = \phi_i\bigl(\theta_i\,
h_{i-1}^\theta(x)\bigr)$ with $h_0^\theta(x) = x$, one then obtains
\begin{align*}
    \|h_i^\theta(x)\|
    &\le C_0\bigl(K_0 + \|\theta_i\|\,\|h_{i-1}^\theta(x)\|\bigr) \\
    &\le C_0 K_0 + C_0 M_0 \|h_{i-1}^\theta(x)\|,
\end{align*}
where $M_0 := \sup_{\theta \in \Theta}\max_i \|\theta_i\| < \infty$ owing
to the compactness of the set $\Theta$. Recursing over
$i = 1, \ldots, L$ gives
\[
    \|h_L^\theta(x)\| \;\le\; C_0 K_0 \sum_{k=0}^{L-1} (C_0 M_0)^k
    \;+\; (C_0 M_0)^L \|x\|,
\]
uniformly in $\theta \in \Theta$. Since the output satisfies
$g_\theta(x) = \theta_{L+1}\, h_L^\theta(x)$, it follows that
\[
    \sup_{\theta\in\Theta} |g_\theta(x)|
    \;\le\; M_0\bigl(K_0' + K_1 \|x\|\bigr) \;=:\; B(x)
\]
for the constants $K_0' := C_0K_0\sum_{k=0}^{L-1}(C_0M_0)^k$ and
$K_1 := (C_0M_0)^L$, which are finite and independent of $\theta$ and $x$.

Finally, for $y \in \{0,1\}$, using the bound
$0 \le \log(1+e^{a}) \le \log 2 + |a|$ for all $a \in \mathbb{R}$, we obtain
\begin{align*}
    |u(\theta, x, y)|
    &\le |y|\,|g_\theta(x)| + \log\!\bigl(1 + e^{g_\theta(x)}\bigr) \\
    &\le 2\,|g_\theta(x)| + \log 2 \\
    &\le 2B(x) + \log 2 \;=:\; G(x,y),
\end{align*}
and hence
\[
    \mathbb E_P\bigl[G(X, Y)\bigr]
    \;\le\; 2M_0 K_1\, \mathbb E_P\|X\| + 2M_0 K_0' + \log 2
    \;<\; \infty
\]
by the assumed first moment condition. This verifies Assumption~\ref{assume:U-LLN3},
and the conditions of Theorem~\ref{thm:consistency} hold.

\medskip

\noindent \textbf{ii) Verification for Convolutional Networks.}
The key observation is that the map
$\theta_i^{\mathrm{conv}} \mapsto T(\theta_i^{\mathrm{conv}})$ is linear. In particular, as a
linear map between finite-dimensional spaces it is continuous. Since each map $\theta_i^{\mathrm{conv}} \mapsto
T(\theta_i^{\mathrm{conv}})$ is continuous and each activation $\phi_i$
is continuous, the composition $\theta \mapsto g_\theta(x)$ is continuous
for every $x$, exactly as in case \textbf{(i)}. This verifies
Assumption~\ref{assume:U-LLN2}.

Since $\theta \mapsto
\|T(\theta_i^{\mathrm{conv}})\|$ is continuous for each $i$ and $\Theta$
is compact,
\[
    M_0^{\mathrm{conv}}
    \;:=\; \sup_{\theta\in\Theta}\, \max_i\,
        \|T(\theta_i^{\mathrm{conv}})\| \;<\; \infty.
\]
Writing $z_i^\theta(x) = \phi_i\bigl(
T(\theta_i^{\mathrm{conv}})\, z_{i-1}^\theta(x)\bigr)$ with
$z_0^\theta(x) = x$, the recursion of case \textbf{(i)} applies verbatim
with $M_0$ replaced by $M_0^{\mathrm{conv}}$, yielding
\[
    \sup_{\theta\in\Theta} |g_\theta(x)|
    \;\le\; M_0^{\mathrm{conv}}\bigl(K_0' + K_1\|x\|\bigr) \;=:\; B(x)
\]
for finite constants $K_0', K_1$ independent of $\theta$ and $x$. The
construction of the integrable envelope $G(x,y) = 2B(x) + \log 2$ and the
verification of Assumption~\ref{assume:U-LLN3} then proceed exactly as in case
\textbf{(i)}.

\medskip

\noindent \textbf{iii) Verification for Transformers.}
Denote by $Z_\ell \in \R^{N\times d}$ the output of the $\ell$th encoder
block, $\ell = 1,\ldots,L$, and for a matrix $Z \in \R^{N\times d}$ write
$\|Z\|_{\infty} := \max_{1\le n\le N}\|Z_{n,:}\|$ for the maximum row
norm. The network output is $g_\theta(x) = w^\top \mathrm{read}(Z_L)$,
where $w$ denotes the final classification head and $\mathrm{read}$ is
the (fixed, linear) map extracting the pooled representation from $Z_L$.

Every constituent map in
Eq.~\eqref{eqn:mha}--\eqref{eqn:ffn} is continuous in $\theta$ for fixed
$x$:
$Z_0\, W_Q^{(h)}\bigl(W_K^{(h)}\bigr)^\top Z_0^\top$ are polynomial in
the parameters, the softmax is smooth, the FFN activation $\phi$ is
continuous, and $\mathrm{LN}$ is smooth since its variance offset is
strictly positive. The embedding $E \mapsto Z_0$ and the classification
head are linear in their respective parameters, and the residual
connections are linear. Hence $\theta \mapsto g_\theta(x)$ is continuous
for every $x$ as a finite composition of continuous maps, verifying
Assumption~\ref{assume:U-LLN2}.

By compactness of $\Theta$,
\[
    M_0 := \sup_{\theta\in\Theta}\,\max_{h,\ell}\,
    \max\Bigl(\bigl\|W_V^{(h,\ell)}\bigr\|,
    \bigl\|W_O^{(h,\ell)}\bigr\|, \|w\|\Bigr) < \infty,
\]
and likewise the layer-normalization parameters are bounded uniformly
over $\Theta$. Two observations drive the bound. First, the softmax in
\eqref{eqn:mha} is applied row-wise, so each row of
$\mathrm{softmax}\bigl(Z_0 W_Q^{(h)}(W_K^{(h)})^\top Z_0^\top /
\sqrt{d_h}\bigr)$ lies in the probability simplex, and each row of the
$h$th summand is therefore a convex combination of the rows of
$Z_0\,W_V^{(h)} W_O^{(h)}$. Since a convex combination of vectors has
Euclidean norm at most the largest of their norms,
\begin{align*}
    \|\mathrm{Attn}(Z_0)\|_{\infty}
    \le\; \sum_{h=1}^{H}\,\max_{m}\
    \bigl\|\bigl(Z_0\,W_V^{(h)} W_O^{(h)}\bigr)_{m,:}\bigr\|
    \le\; H M_0^2\,\|Z_0\|_{\infty}.
\end{align*}
In particular, the query and key matrices shape only the mixing weights
and play no role in the magnitude bound.
Second, $\mathrm{LN}$ produces uniformly bounded outputs: for any
$v \in \R^d$, the vector
$u := (v - \bar v\mathbf{1})\big/\bigl(\tfrac{1}{d}\|v - \bar v
\mathbf{1}\|^2 + \epsilon\bigr)^{1/2}$ satisfies
$\|u\|^2 = \|v - \bar v\mathbf{1}\|^2 \big/ \bigl(\tfrac{1}{d}\|v -
\bar v\mathbf{1}\|^2 + \epsilon\bigr) \le d$, and hence
\[
    \|\mathrm{LN}(v)\| \;\le\; \Gamma \;<\; \infty
    \qquad \text{uniformly in } v \in \R^d \text{ and }
    \theta \in \Theta,
\]
where $\Gamma$ depends only on $d$ and the (compactly constrained)
layer-normalization parameters. Since both \eqref{eqn:res1} and \eqref{eqn:ffn}
terminate in $\mathrm{LN}$, it follows that
$\|Z_\ell\|_{\infty} \le \Gamma$ for every $\ell \ge 1$, uniformly in
$\theta \in \Theta$ and in the input: in contrast to cases \textbf{(i)}
and \textbf{(ii)}, no recursion accumulates, as each block's output
bound is independent of its input bound. Since $\mathrm{read}$ returns a
convex combination (or a coordinate projection) of the rows of $Z_L$,
\[
    \sup_{\theta\in\Theta}\, |g_\theta(x)|
    \;\le\; \sup_{\theta\in\Theta}\|w\| \cdot \|Z_L\|_{\infty}
    \;\le\; M_0\,\Gamma \;=:\; B,
\]
a finite constant independent of both $\theta$ and $x$. The envelope
$G(x,y) := 2B + \log 2$ constructed as in case \textbf{(i)} is therefore
itself constant, so that $\mathbb E_P[G(X,Y)] < \infty$ holds with no moment
condition on $X$ whatsoever. This verifies Assumption~\ref{assume:U-LLN3}, and the
conditions of Theorem~\ref{thm:consistency} hold.
\end{proof}

\end{document}